\documentclass[letterpaper, 10 pt, conference]{ieeeconf}  

\IEEEoverridecommandlockouts                              
\usepackage{graphicx} 
\usepackage{cite}
\usepackage[utf8]{inputenc}
\usepackage{amsmath}
\usepackage[ruled, vlined, linesnumbered]{algorithm2e}
\usepackage{amssymb}
\usepackage{amsfonts}
\usepackage[bb=boondox]{mathalfa}
\usepackage[utf8]{inputenc}
\usepackage[english]{babel}
\usepackage{bm}
\usepackage{accents}
\usepackage[dvipsnames]{xcolor}
\usepackage{soul}
\usepackage{physics}
\usepackage[caption=false,font=footnotesize]{subfig}

\newcommand{\mc}{\mathcal}

\newcommand{\argmin}{\mathop{\rm argmin}\limits}

\newtheorem{theorem}{\bf Theorem}
\newtheorem{assumption}{\bf Assumption}

\newtheorem{proposition}{\bf Proposition}

\title{\LARGE \bf
Dynamic Allocation of Visual Attention for Vision-based Autonomous Navigation under Data Rate Constraints}

\author{ Ali Reza Pedram$^1$ \and Riku Funada$^2$ \and Takashi Tanaka$^3$ 
\thanks{*This work is supported by DARPA Grant D19AP00078, FOA-AFRL-AFOSR-2019-0003, NSF Award 1944318, and JSPS KAKENHI Grand Number 21K20425.}
\thanks{$^{1}$ Walker Department of Mechanical Engineering, University of Texas at Austin. {\tt\small apedram@utexas.edu}.
        $^2$ Department of Systems and Control, Tokyo Institute of Technology. {\tt\small funada@sc.e.titech.ac.jp}.
         $^{3}$Department of Aerospace Engineering and Engineering Mechanics, University of Texas at Austin.
        {\tt\small ttanaka@utexas.edu}. }%
}

\begin{document}

\maketitle
\thispagestyle{empty}
\pagestyle{empty}

\begin{abstract}
This paper considers the problem of task-dependent (top-down) attention allocation for vision-based autonomous navigation using known landmarks. Unlike the existing paradigm in which landmark selection is formulated as a combinatorial optimization problem, we model it as a resource allocation problem where the decision-maker (DM) is granted extra freedom to control the degree of attention to each landmark. The total resource available to DM is expressed in terms of the capacity limit of the in-take information flow, which is quantified by the directed information from the state of the environment to the DM’s observation. We consider a receding horizon implementation of such a controlled sensing scheme in the Linear-Quadratic-Gaussian (LQG) regime. The convex-concave procedure is applied in each time step, whose time complexity is shown to be linear in the horizon length if the alternating direction method of multipliers (ADMM) is used. Numerical studies show that the proposed formulation is sparsity-promoting in the sense that it tends to allocate zero data rate to uninformative landmarks.
\end{abstract}

\section{INTRODUCTION}
\label{sec:intro}
It is generally believed that human brain does not have sufficient computational throughput to process raw visual input entirely \cite{marois2005capacity}. Even though the capacity of the optimal nerve from retina to the early stage cortical areas is estimated to be about $10^7$ bits per second (bps), only a small fraction of it is known to be processed further (Fig.~\ref{fig:vision_control}). This limitation is known as the \emph{attention bottleneck} \cite{zhaoping2014understanding}, whose capacity is estimated to be as small as $100$ bps \cite{sziklai1956some}. Since high-level decisions by downstream brain areas must rely on this limited information, the information content transmitted through the attention bottleneck has to be carefully selected so as to best assist the task to be completed by the decision-maker (DM). Task-dependent (top-down) attention mechanisms are believed to play major roles in this information selection. 

Although the task-dependent nature of human attention (e.g., eye movements) is well-known, theoretical understanding of top-down visual attention lags far behind the understandings of saliency-based (bottom-up) counterpart. This is firstly because top-down attention involves a larger brain areas including downstream components of the visual pathway (e.g., extrastriate cortex) whose functionalities are far less understood compared to early-stage components (e.g., primary visual cortex). Another reason is the lack of a theoretical framework that formally translates the notion of task-relevance into a mathematical language. Unlike bottom-up attention, top-down attention must be understood in the dynamic interaction between the DM and the environment, and modeling such an interaction itself is a challenging task. Indeed, even though methodologies for predicting bottom-up attention (e.g., computing a saliency map \cite{itti2000saliency} from an input image) are widely available today, computational tools for predicting top-down attention are still very limited.

To partially fill the lack of the theoretical framework for top-down visual attention, this paper formulates the problem of dynamic allocation of visual attention as a data rate constrained optimal controlled sensing problem. Our goal is to provide a mathematical metric to impose data rate constraints on the attention bottleneck and to develop a preliminary algorithm for attention allocation (i.e., allocation of data rate on landmarks in a visual scene). We analyze the impact of the capacity of the attention bottleneck on the visual attention in the context of a simple vision-based navigation scenario using known landmarks.

\begin{figure}
    \centering
    \includegraphics[width = \columnwidth]{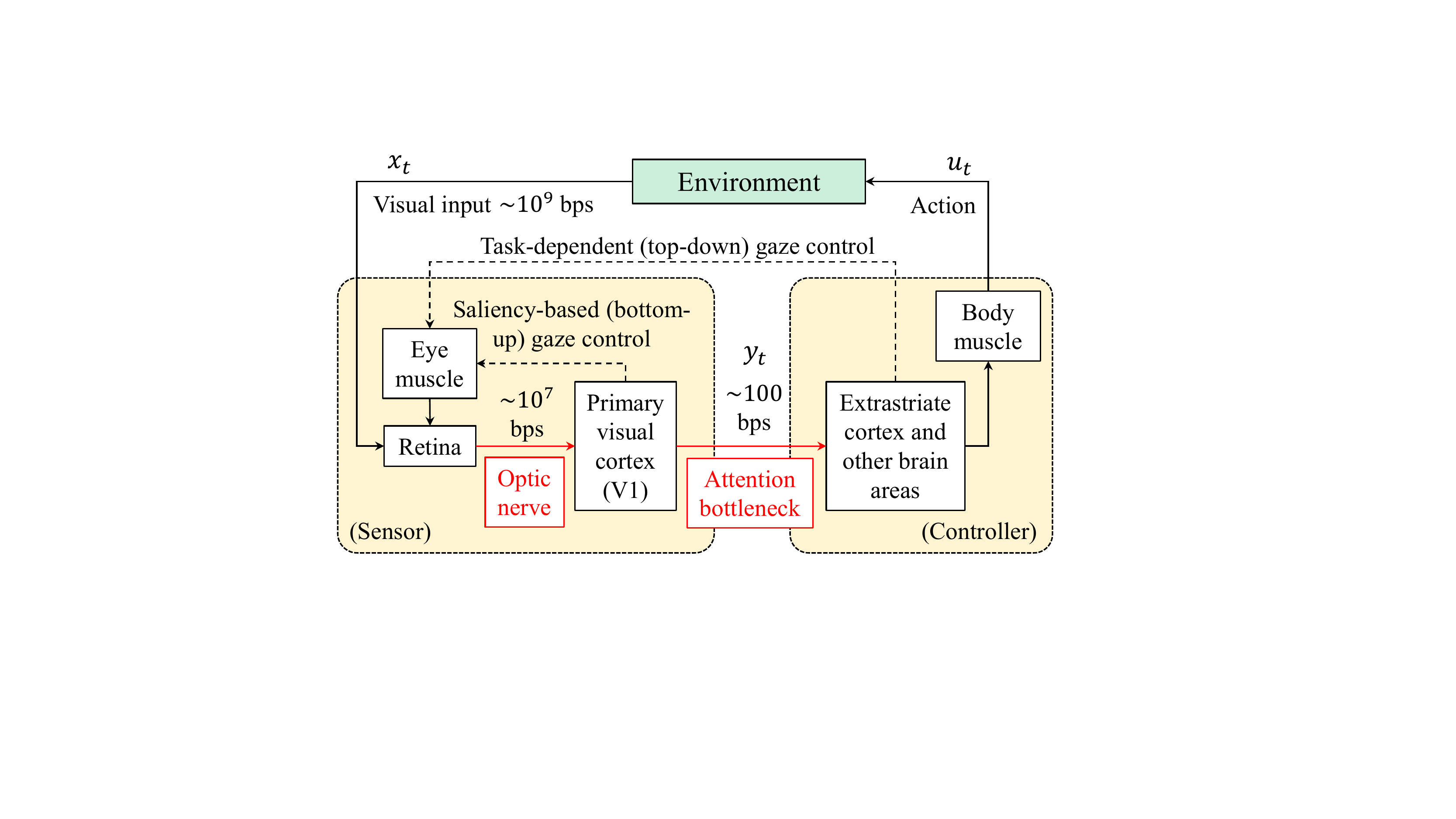}
    \vspace{-3ex}
    \caption{Controlled sensing architecture of human visual attention \cite{zhaoping2014understanding}.}
    \label{fig:vision_control}
    \vspace{-3ex}
\end{figure}

\subsection{Attention allocation and directed information}

Fig.~\ref{fig:vision_control} can be viewed as a controlled sensing architecture with a data rate constraint at the sensor output, which has been actively studied in the networked control systems (NCS) literature. 
Consequently, our approach is closely related to the method for simultaneous sensor-controller synthesis for minimum information control studied in the NCS literature. 
Assuming that $\bm{x}_{1:T}\triangleq\{\bm{x}_t\}_{t=1, 2, ... , T}$ is the state random process of the environment (plant) and $\bm{y}_{1:T}\triangleq\{\bm{y}_t\}_{t=1, 2, ... , T}$ is the random process of the sensor output, prior work \cite{silva2010framework,tanaka2016rate,kostina2019rate} has shown that the information content of $\bm{y}_{1:T}$ can be compressed by means of entropy coding up to 
\begin{equation}
\label{eq:di_def}
    I(\bm{x}_{1:T}\rightarrow \bm{y}_{1:T})\triangleq \sum_{t=1}^T I(\bm{x}_{1:t};\bm{y}_t|\bm{y}_{1:t-1}) \qquad \text{[bits].}
\end{equation}
The quantity \eqref{eq:di_def} is known as directed information (DI) \cite{massey1990causality}. Therefore, the optimal sensor-controller pair to minimize the control cost $\sum_{t=1}^T \mathbb{E}[J(\bm{x}_t,\bm{u}_t)]$ subject to the data rate constraint $R$ bits is characterized by
\begin{subequations}
\label{eq:sensor_controller_synthesis}
\begin{align}
    \min_{\text{sensor, controller}} &\qquad \sum\nolimits_{t=1}^T \mathbb{E}[J(\bm{x}_t,\bm{u}_t)] \\
    \text{s.t.}\qquad&\qquad I(\bm{x}_{1:T}\rightarrow \bm{y}_{1:T}) \leq R.
\end{align}
\end{subequations}
\begin{figure}
    \centering
    \includegraphics[width = \columnwidth]{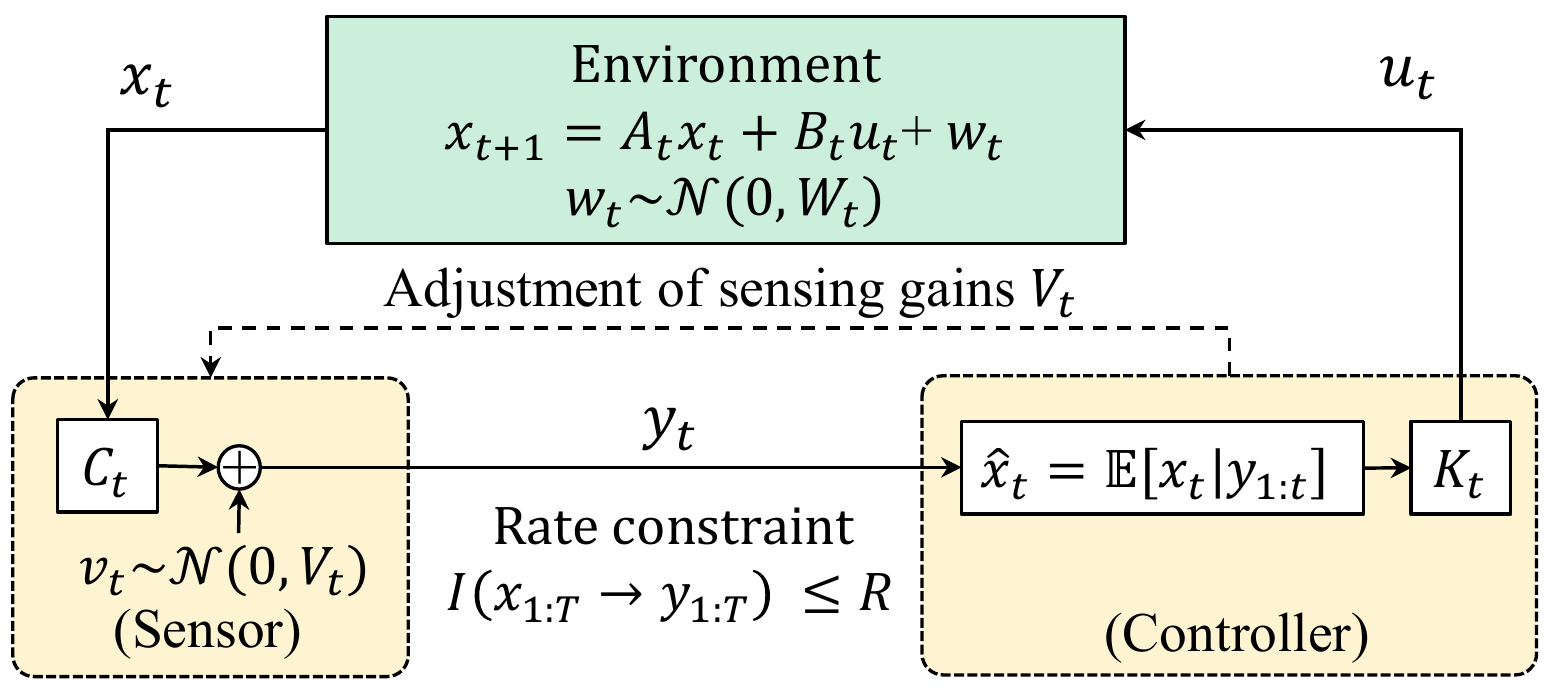}
    \vspace{-2ex}
    \caption{Controlled sensing of linear Gaussian systems.}
    \label{fig:sensor_control}
    \vspace{-3ex}
\end{figure}
To obtain useful insights on top-down visual attention through the mathematical optimization \eqref{eq:sensor_controller_synthesis}\footnote{Although currently there is no evidence that \eqref{eq:sensor_controller_synthesis} is the guiding principle of human top-down attention, synthesizing visual perception strategies under \eqref{eq:sensor_controller_synthesis} and comparing them with experimental data (e.g., eye tracker data) would be an interesting future study. Instead of \eqref{eq:sensor_controller_synthesis}, we also consider its soft-constrained version \eqref{eq:sparse-idea} below.}, in this paper, we apply \eqref{eq:sensor_controller_synthesis} to a simple task for DM to move along a given trajectory using bearing observation of known landmarks.
We consider the situation where the nonlinear dynamics of the environment (DM's position) can be approximated by a linear time-varying (LTV) system with Gaussian noise as shown in Fig.~\ref{fig:sensor_control}. We assume that DM's task quality is measured by a quadratic function of the state (square deviation of DM's location from the reference).\footnote{The problem \eqref{eq:sensor_controller_synthesis} in the LQG regime has been studied in the past \cite{tanaka2017lqg}, where it is shown that the optimal sensor-controller pair is given by a linear sensor $\bm{y}_t=C_t\bm{x}_t+\bm{v}_t, \bm{v}_t\sim\mathcal{N}(0,V_t)$ followed by a linear controller $\bm{u}_t= K_t\mathbb{E}[\bm{x}_t|\bm{x}_{1:t}]$, where sensor gain $C_t$, diagonal covariance $V_t$ and controller gain $K_t$ are obtained by a certain convex optimization. However, 
the problem set up in this paper is different from \cite{tanaka2017lqg} in that the matrix $C_t$ cannot be freely chosen.}
DM perceives its current location via the sensing mechanism
\begin{equation}
\label{eq:sensor}
\bm{y}_t=C_t\bm{x}_t+\bm{v}_t, \quad \bm{v}_t \sim \mathcal{N}(0, \text{diag}(V_{1,t}, \dots , V_{M,t}))
\end{equation}
where the $i$-th row of \eqref{eq:sensor} is the information that DM obtains by observing the $i$-th landmark. The matrix $C_t$ is determined by the relative positions of landmarks with respect to the DM. However, the DM is granted freedom to tune $V_{i,t}$ to optimally allocate attention resources. That is, setting $V_{i,t}$ to its lower bound corresponds to observing the landmark $i$ with the highest possible accuracy, while setting $V_{i,t}=\infty$ means zero data rate is allocated to that landmark.   

\subsection{Related work and limitation}

Decision making under sensing resource constraints is a ubiquitous problem in autonomy and has been widely recognized by control and robotics communities. The problems of sensor selection \cite{joshi2008sensor,zhang2017sensor,Hashemi21_greedy_sensor}, sensor scheduling \cite{gupta2006stochastic,vitus2012efficient,chepuri2014sparsity}, and landmark selection for autonomous navigation (e.g., \cite{carlone2018attention}) have been actively studied in various setups. A popular approach is to impose a cardinality constraint on sensing modalities (e.g., the number of sensors/landmarks that can be simultaneously engaged). Although a cardinality constraint is conceptually simple, it may not be appropriate to model human visual attention as a discrete choice problem.


\subsection{Contribution of this paper}
Contributions of this paper are summarized as follows:

\begin{itemize}
    \item We formulate the task-dependant visual attention allocation problem under data rate restriction as an optimization problem, where the restriction is imposed by the DI from the state of the environment to the observations. The formulation is predictive and it is developed in the receding horizon fashion.
 \item Noticing that the formulated problem belongs to the class of difference-of-convex programs \cite{lipp2016variations}, we apply the convex-concave procedure and present an algorithm which finds a local minimum in polynomial time. Using ADMM, we propose a distributed algorithm which reduces the computational complexity from being cubic to linear in horizon length.
\item By numerical studies, we show that the proposed formulation is sparsity-promoting in the sense that it tends to allocate zero data rate to the uninformative landmarks.
We provide insights into the observed sparsity-promoting property by analyzing a simple special case of the proposed formulation. 
%
%
\end{itemize}


\subsection{Notation}
Throughout the paper, we use lower case boldface symbols such as $\bm{x}$ to represent random variables (r.v.), while $x$ is a realization of $\bm{x}$. The matrices are represented by upper case symbols and $\|.\|_F$ is Frobenius norm. $X_{i:j}$ is used to denote the sequence $\{X_i, X_{i+1}, \dots, X_j\}$. $h(\bm{y})\triangleq  -\int p(y) \log p(y)dy$ is used to denote the differential entropy of the r.v. $\bm{y}$. For a Gaussian r.v. $\bm{y} \sim \mathcal{N} (y, \Sigma_y)$, $h( \bm{y}) = \frac{1}{2}\log\det(2\pi e \Sigma_y)$. Mutual information between random variables $\bm{x}$ and $\bm{y}$ is denoted by $I(\bm{x};\bm{y}) \triangleq h(\bm{x})-h(\bm{x}|\bm{y})$. $[i,j]$ denotes the set $\{i, i+1,  \dots, j\}$ and $[i]\triangleq [1,i]$.

\section{Problem Formulation}
\label{sec:pre}
In this paper, we consider a DM whose task is to follow a reference trajectory $x_{1:T}^{ref}$ with minimum deviation based on the observation of known landmarks existing in the environment. We assume the DM's attention bottleneck has a data rate limit. We formulate an optimal landmark selection problem, based on which the DM strategically allocates its attention to the landmarks under this limited data rate in order to achieve the highest quality of the task.

The DM's dynamics are described by a generic non-linear discrete-time state-space model
\begin{equation}
\label{eq:dyn}
\bm{x}_{t+1}^{act}= f(\bm{x}_{t}^{act}, \bm{u}_{t}^{act}) +  \bm{w}_t, \quad \bm{w}_t \overset{i.i.d.}{\sim} \mathcal{N}(0, W),
\end{equation}
where $\bm{x}_t^{act} \in \mathbb {R}^n$ is the actual state vector (location) of the DM, with the initial distribution $\bm{x}_1^{act} \sim \mathcal{N}(0, P_{1|0})$. The DM estimates its state by performing visual measurements
of landmarks whose locations are known a~priory. By observing all the landmarks with full accuracy, the DM receives a sensor measurement
\begin{equation}
\label{eq:nl_sensor}
    \bm{y}_{t}^{act}= h(\bm{x}_{t}^{act})+\bm{v}_{t}, \quad  \bm{v}_{t} \overset{i.i.d.}{\sim} \mathcal{N}(0, \hat{V}),
\end{equation}
where the $i$-th row of $\bm{y}_{t}^{act}$ is the sensor measurement of the $i$-th landmark, and  $\hat{V} \triangleq {\rm diag}(\{\hat{V}_{i}\}_{i \in [M]})$ is the measurement noise covariance characterizing the limitation of the sensing accuracy.

Although all the measurements (\ref{eq:nl_sensor}) are always available, the DM may not have sufficient information throughput to process their information content entirely. In this case, the DM may need to reduce the data rate allocated to some of the landmarks. From  information-theoretic perspective (as detailed in this sequel), this can be done by choosing  $V_{i,t} \succeq \hat{V}_{i}$ to  each landmark $i$ at time $t$.





We assume the deviation from $x_{1:T}^{ref}$ stays small because of the feed-back controller the DM uses for the path following. Based on this assumption, the DM's dynamics can be approximated by a linearized model around $x_{1:T}^{ref}$. We drive  the approximated linear model under the following assumption.
\begin{assumption}
The state sequence $x_{1:T}^{ref}$ is a feasible target trajectory for (\ref{eq:dyn}), meaning that there exist a control input sequence $u_{1:T-1}^{ ref}$ such that $\forall t\in[T-1]$ we have $x_{t+1}^{ref} = f(x_{t}^{ref}, u_{t}^{ref})$.
\end{assumption}

Based on the feasibility of path $x_{1:T}^{ref}$, the non-linear dynamics expressed in (\ref{eq:dyn}) and (\ref{eq:sensor}) is linearized as
\begin{equation}
\label{eq:dyn-lin}
\begin{split}
    \bm{x}_{t+1} &= A_t \bm{x}_{t}+ B_t  \bm{u}_{t}+  \bm{w}_t, \quad \bm{w}_{t} \overset{i.i.d.}{\sim} \mathcal{N}(0, W),\\
    \bm{y}_{t}&= C_t \bm{x}_{t} +\bm{v}_{t}, \quad  \bm{v}_{t}, \overset{i.i.d.}{\sim} \mathcal{N}(0, V_t),
\end{split}
\end{equation}
 where $\bm{x}_{t}\triangleq \bm{x}_{t}^{act} - x_{t}^{ref}$, $\bm{u}_{t}= \bm{u}_{t}^{act} - u_{t}^{ref}$, and $\bm{y} \triangleq \bm{y}_{t}^{act}-h_i(x_{t}^{ref})$ are the deviation from the reference values.
 The linear system (\ref{eq:dyn-lin}) starts from $\bm{x}_1 \sim \mathcal{N}(-x_{1}^{ref}, P_{1|0})$, and we have $A_t := \nabla_{x} f(x,u)|_{(x_{t}^{ref}, u_{t}^{ref})}$, $B_t := \nabla_{u} f(x,u)|_{(x_{t}^{ ref}, u_{t}^{ref})}$, and $C_t=[C_{1,t}^{\top}, \dots, C_{M,t}^{\top}]^{\top}$, where $C_{i,t}:= \nabla_{x} h_{i}(x)|_{x_{t}^{ref}}$.

\subsection {Kalman Filter}
\label{sec:EKF}

 


The DM at time $k$ observes the measurement $\bm{y}_{k}$ to compute the least mean square error (MSE) prediction $\hat{\bm{x}}_{k|k-1}:= \mathbb{E}[\bm{x}_k| \bm{y}_{1:k-1}]$ and the least  MSE filtered estimate  $\hat{\bm{x}}_{k|k}:= \mathbb{E}[\bm{x}_k| \bm{y}_{1:k}]$. Using the Kalman filter (KF) associated with (\ref{eq:dyn-lin}), these estimates are computed recursively by $\hat{x}_{k|k-1}= A_{k-1} \hat{x}_{k-1|k-1}+ B_{k-1} u_{k-1}$ and $ \hat{x}_{k|k}=\hat{x}_{k|k-1}+ L_{k} (y_{k}-C_{k} \hat{x}_{k|k-1})$. Here, $L_{k}$ is the Kalman gain computed as
\begin{equation*}
    L_{k}= P_{k|k-1} C_{k}^\top (C_{k} P_{k|k-1} C_{k}^\top+ V_{k})^{-1},
\end{equation*}
where $P_{k|k-1}\triangleq \textbf{Cov}(\bm{x}_{k}- \hat{\bm{x}}_{k|k-1})$ and  $P_{k|k}\triangleq \textbf{Cov} (\bm{x}_{k}- \hat{\bm{x}}_{k|k})$ represent the predicted and filtered error covariances, respectively.  These covariances are computed from Riccati recursion 
\begin{subequations}
\label{eq:riccati}
\begin{align}
\label{eq:cov_update}
    P_{k|k}^{-1}=&P_{k|k-1}^{-1}+ C_{k}^\top V_{k}^{-1}C_{k},\\ 
    P_{k+1|k}=& A_k P_{k|k} A_k^\top+ W.
\end{align}
\end{subequations}
Here, we emphasize that $P_{k|k}$ and $P_{k+1|k}$ are controllable by the DM via the decision variable $V_k$ in our problem formulation. In what follows, we formulate an optimization problem in terms of $V_k$ and Fisher information matrices $Q_{k|k-1}\triangleq P_{k|k-1}^{-1} $ and 
$Q_{k|k}\triangleq P_{k|k}^{-1}= P_{k|k-1}^{-1}+C_{S_k}^\top V_{k}^{-1}C_{k} $.

\subsection{Control Scheme}
\label{sec:LQR}

 
Suppose the DM's task is to minimize the deviation from $x_{1:T}^{ref}$ by implementing the control policy which minimizes the quadratic control cost 
$ J_{1:T} \triangleq \sum_{t=1}^{T-1} \mathbb{E} \big[\|\bm{x}_{t+1}\|^2_Q+ \|\bm{u}_t\|^2_R]$,
with some given $R\succ 0$ and $ Q \succ 0$. Formally, this LQG control problem is formulated as
\begin{equation}
    \begin{split}
        \min_{\{\bm{u}_t\}_{t=1}^{T-1}}\quad  &  \sum_{t=1}^{T-1} \mathbb{E} [ \|\bm{x}_{t+1}\|^2_Q+ \|\bm{u}_t\|^2_R ]\\
    \end{split}
\end{equation}
subject to (\ref{eq:dyn-lin}), where the minimization is performed over the space of causal policies $\bm{u}_t=U_{t}(\bm{y}_{1:t})$. By the separation principle, the optimal controller is $\bm{u}_t=K_t 
\hat{\bm{x}}_{t|t}$, where 
\begin{align} \label{eq:opt_control}
K_t = -(B_t^\top S_{t} B_t+R)^{-1}(B_t^\top S_{t} A_t).
\end{align}
In (\ref{eq:opt_control}), $S_t$ is computed iteratively backward in time by the dynamic Riccati equation:
\begin{align*}
S_{t-1}\!=\!A_t^\top S_{t} A_t\!-\!A_t^\top\!S_{t} B_t (B_t^\top S_{t} B_t+R)^{-1}\!B_t^\top S_{t} A_t\!+\!Q,
\end{align*}
with the terminal condition $S_T=Q$. The optimal control cost is 
\begin{equation}
\label{eq:lqg_cost}
    J_{1:T}= \sum_{t=1}^T Tr(\Theta_t P_{t|t})+J^{c}_{1:T},
\end{equation}
where $\Theta_t= K_t^\top (B_t^\top S_t B_t+ R) K_t$. In (\ref{eq:lqg_cost}), $J^{c}_{1:T} \triangleq \sum_{t=1}^T Tr(S_t W)+ Tr((S_0-Q)P_{1|0})$ is a constant term which is independent of the attention allocation the DM makes, and thus it is neglected in the rest of the paper.





\section{Proposed Formulation} \label{sec:proposed_form}

\subsection{Regularization with DI}

In this paper,  we study the landmark selection problem for a receding horizon $H$. More precisely, the DM at time $t$ starts from the initial information $Q_{t|t-1}$, and seeks the optimal set of attention allocation $V_{t:t+H}$ to minimize control cost incurs from time $t$ up to $t+H$ i.e.,  $J_{t:t+H}= \sum_{k=t}^{t+H} Tr(\Theta_k Q_{k|k}^{-1})$,  while keeping the required data rate small.


As we discussed in Section~\ref{sec:intro}, the DI $I(\bm{x}_{t:t+H}\rightarrow \bm{y}_{t:t+H}| \bm{y}_{t-1})$\footnote{Here, we adopt conditional DI as in our receding horizon implementation, the random variable $\bm{y}_{1:t-1}$ is given at time step $t$.}
quantifies the required data rate to transfer the information content of $\bm{y}_{t:t+H}$ about $\bm{x}_{t:t+H}$. Based on this interpretation, we propose to regularize $J_{t:t+H}$ with $I(\bm{x}_{t:t+H}\rightarrow \bm{y}_{t:t+H}| \bm{y}_{1:t-1})$, and formulate the problem of attention allocation under data rate constraint as
\begin{equation}
\label{eq:sparse-idea}
    \begin{split}
        \min_{V_{t:t+H}} \quad & J_{t:t+H} +\beta I( \bm{x}_{t:t+H} \rightarrow \bm{y}_{t:t+H} |\bm{y}_{1:t-1}).
    \end{split}
\end{equation} 
Problem (\ref{eq:sparse-idea}) is the soft-constrained version of (\ref{eq:sensor_controller_synthesis}), where $\beta$ is the Lagrange multiplier. A greater $\beta$ places more weight on the data rate and reduces the data rate of the optimal attention allocation $V^*_{t:t+H}$. The DM only implements $V^*_{t}$ (i.e, the first element of the sequence $V^*_{t:t+H}$), and the landmark selection continues by solving (\ref{eq:sparse-idea}) again at $t+1$.





We now rewrite (\ref{eq:sparse-idea}) using information matrices $Q_{k|k-1}$ and $Q_{k|k}$.
The DI term in (\ref{eq:sparse-idea}) can be written as:
    \begin{align}
        \nonumber   
        &I(\bm{x}_{t:t+H}\rightarrow \bm{y}_{t:t+H}| \bm{y}_{1:t-1}) = \sum_{k=t}^{t+H} I(\bm{x}_{t:k};\bm{y}_{k}| \bm{y}_{1:k-1})\\ \nonumber 
        &= \sum_{k=t}^{t+H} I(\bm{x}_k;\bm{y}_{k}|  \bm{y}_{1:k-1})+
        I(\bm{x}_{t:k-1};\bm{y}_{k}|\bm{x}_k,  \bm{y}_{1:k-1}) \\ \nonumber 
        &\overset{(a)}{=}\sum_{k=t}^{t+H} \!I(\bm{x}_k;\bm{y}_{k}| \bm{y}_{1:k-1})=\sum_{k=t}^{t+H}  h(\bm{x}_k|  \bm{y}_{1:k-1})-h(\bm{x}_k|  \bm{y}_{1:k})\\  \label{eq:DI}
        &=  \sum_{k=t}^{t+H} \frac{1}{2} \log \det Q_{k|k}- \frac{1}{2}\log\det Q_{k|k-1}.
    \end{align}
In step (a), we used the fact that $\bm{x}_{t:k-1}$ and $\bm{y}_k$ are conditionally independent given $\bm{x}_k$.




Introducing a change of variables $U_k\triangleq V_k^{-1}$ and substituting (\ref{eq:DI}), (\ref{eq:sparse-idea}) becomes
\begin{subequations}
\label{eq:exp-prob}
\begin{align}
\nonumber
    \min \quad &\sum_{k=t}^{t+H}Tr(\Theta_k Q_{k|k}^{-1})\\ \label{eq:obj}
     & \quad + \frac{\beta}{2}(\log\det Q_{k|k}-\log\det Q_{k|k-1})\\ \label{eq:cons_b}
    \text{s.t.} \quad 
    & Q_{k|k}= Q_{k|k-1} + C_{k}^\top U_{k} C_{k},\\ \label{eq:cons_c}
    & Q_{k|k-1}^{-1} = A_{k-1} Q_{k-1|k-1}^{-1} A_{k-1}^\top+ W,\\ \label{eq:cons_d}
    & U_k \preceq \hat{V}^{-1},
\end{align}
\end{subequations}
where the decision variables are $Q_{k|k}$, $Q_{k|k-1}$, and $U_k$ for $k\in[t,t+H]$. The constraints (\ref{eq:cons_b}) and (\ref{eq:cons_d}) are imposed for $ k \in [t,t+H]$ while the constraint (\ref{eq:cons_c}) is imposed for all  $ k \in [t+1, t+H]$ with the initial condition $Q_{t|t-1}=P_{t|t-1}^{-1}$.

In (\ref{eq:exp-prob}), all the constraints except (\ref{eq:cons_c}) are convex. In the following proposition, we show (\ref{eq:cons_c}) can be replaced  by a convex inequality constraint without changing the nature of the problem. More precisely, the following problem can be solved instead of (\ref{eq:exp-prob}).
\begin{subequations}
\label{eq:relaxed_prob}
\begin{align}
    \min \quad & \mathrm{(\ref{eq:obj})} \\ 
    \text{s.t.} \quad \label{eq:not_LMI}
    & Q_{k|k-1}^{-1} \succeq A_{k-1} Q_{k-1|k-1}^{-1} A_{k-1}^\top+ W,\\ & \mathrm{(\ref{eq:cons_b})}\; \text{and}\; \mathrm{(\ref{eq:cons_d})}.
\end{align}
\end{subequations}

\begin{proposition}
\label{prop:one}
Let $J^*_1$ and $(U^*_k, Q^*_{k|k-1}, Q^*_{k|k})$ be the optimal value and optimal solution of problem (\ref{eq:relaxed_prob}), respectively. Then, the optimal value of (\ref{eq:exp-prob}) is $J^*_2=J^*_1$ and the optimal solution of (\ref{eq:exp-prob}) is $(U_k^*,Q^{**}_{k|k-1},Q^{**}_{k|k})$, where $Q^{**}_{k|k-1}$ and $Q^{**}_{k|k}$ are recursively defined by
\begin{subequations}
\begin{align}
    Q^{**}_{k|k}=& Q^{**}_{k|k-1}+  C_k^\top U^{*}_k C_k,\\
    Q_{k|k-1}^{**-1}=& A_{k-1} Q_{k-1|k-1}^{**-1}A_{k-1}^\top+ W,
\end{align}
\end{subequations}
starting from $Q^{**}_{t|t-1}=Q^{*}_{t|t-1}$.
\end{proposition}

\begin{proof}
See \cite[Proposition~1]{Jung2021optimal} which provides a proof for similar relaxation.
\end{proof}

Convexity of (\ref{eq:not_LMI}) can be seen by noticing that it can be rewritten as
an equivalent linear matrix inequality. Thus (\ref{eq:relaxed_prob}) becomes
\begin{subequations}
\label{eq:prob_new}
\begin{align}
\nonumber
    \min \quad & \sum_{k=t}^{t+H} Tr(\Theta_k Q_{k|k}^{-1}) \\ \label{eq:prob_new_obj}
    & \quad + \frac{\beta}{2} (\log \det Q_{k|k}- \log\det Q_{k|k-1})\\ \label{eq:prob_new_cons1}
    \text{s.t.} \quad 
    & Q_{k|k}= Q_{k|k-1}+  C_{k}^\top U_{k} C_{k},\\  \label{eq:prob_new_cons3}
    & \begin{bmatrix}
    Q_{k|k-1}& Q_{k|k-1} A_{k-1}& Q_{k|k-1}W^{\frac{1}{2}}\!\\
    A_{k-1}^\top Q_{k|k-1}& Q_{k-1|k-1}& 0\\
    W^{\frac{1}{2}} Q_{k|k-1}& 0 & I
    \end{bmatrix} \!\!\succeq 0,\\ \label{eq:prob_new_cons4}
    & U_{k} \preceq \hat{V}^{-1}.
\end{align}
\end{subequations}
Unfortunately,  $\log\det(Q_{k|k})$ in  (\ref{eq:prob_new_obj}) is a non-convex function and currently it is not known to the authors if (\ref{eq:prob_new}) can be formulated as a convex optimization problem. 
In the Section~\ref{sec:CCP}, we developed an algorithm that finds a local minimum of (\ref{eq:prob_new}) in polynomial time.

\subsection{Algorithm}
\label{sec:CCP}

Problem (\ref{eq:prob_new}) is an instance of difference of convex (DC) programming for which we can use CCP algorithm \cite{lipp2016variations} to find a local minimum. The CCP is an iterative algorithm which starts from an initial feasible solution of the DC program at iteration $j=0$. At iteration $j+1$, the non-convex terms of the DC problem are over-approximated by their affine approximation computed  by linearization around the solution of iteration $j$.  In each iteration, the approximate program is convex and can be solved efficiently.
The process is repeated until a locally optimum solution is found.
It is shown in \cite{lipp2016variations} that if the initial iteration of CCP algorithm is feasible all subsequent iterations will be  feasible, and the algorithm monotonically converges to a local minimum.

 The only non-convex term in (\ref{eq:prob_new}) is $g_0(Q_{k|k})=\log\det Q_{k|k}$ which is upper-approximated by the affine $\hat{g}_0 = \log\det Q_{k|k,j}+Tr(Q_{k|k,j}^{-1}Q_{k|k})-Tr(Q_{k|k,j}^{-1} Q_{k|k,j})$. Hence, at the $j+1$-th CCP iteration, we solve
\begin{subequations}
\label{eq:LMI}
\begin{align}
\nonumber
    \min \quad & \sum_{k=t}^{t+H} Tr(\Theta_k Q_{k|k}^{-1}) \\
    & \quad + \frac{\beta}{2} ( Tr(Q_{k|k,j}^{-1}Q_{k|k}) - \log\det Q_{k|k-1})\\
    \text{s.t.} \quad & \mathrm{(\ref{eq:prob_new_cons1})}-\mathrm{(\ref{eq:prob_new_cons4})}.
\end{align}
\end{subequations}
with decision variables $U_{k}$, $ Q_{k|k}$, and $ Q_{k|k-1}$ for $k \in [t, t+H]$. Problem (\ref{eq:LMI}) is a max-det problem, whose computational complexity is typically $\mathcal{O}(H^3)$.

For problems with long time horizons, the computational complexity of $\mathcal{O}(H^3)$ is not acceptable for online implementations of the algorithm. In Section~\ref{sec:ADMM}, we exploit the sparsity pattern in (\ref{eq:LMI}) and propose a distributed CCP algorithm.




\section{Distributed Algorithm Using ADMM}
\label{sec:ADMM}
In this section, we derive a distributed algorithm for problem (\ref{eq:LMI}) by exploiting the fact that variables at each step are only coupled with variables at the previous and the next time steps. We propose a method to resolve the coupling between different time steps and derive an algorithm using alternating direction method of multipliers (ADMM) \cite{boyd2011distributed}. We show this algorithm facilitates solving this problem with the time complexity $\mathcal{O}(H)$.

\subsection{Review of ADMM}
The ADMM is applicable to the constrained problem
\vspace{-0.2cm}
\begin{equation}
\label{eq:admm_form}
    \begin{split}
    \min \;\;& f(X)+g(Z)\\
    \text{s.t.}\;\; & X=Z,
    \end{split}
\end{equation}
with variable $X \in \mathbb{S}^n,$ and $ Z \in \mathbb{S}^n$, where $f$ and $g$ are convex functions. The augmented Lagrangian with scaled dual variable $D \in \mathbb{S}^{n}$ for problem (\ref{eq:admm_form}) is defined as  
\begin{equation*}
    L_{\rho}(X, Z, D)= f(X)+g(Z)+\frac{\rho}{2}\|X-Z+D\|_{F}^2,
\end{equation*}
where $\rho \in \mathbb{R}^{+}$ is a penalty parameter. The $l+1$-th ADMM iteration for this problem is 
\begin{align*}
    &X^{l+1}:=\argmin_{X} \big( f(X)+\frac{\rho}{2}\|X-Z^l+D^l\|_F^2 \big),\\
    &Z^{l+1}:= \argmin_Z \big( g(Z)+\frac{\rho}{2}\|X^{l+1}-Z+D^l\|_F^2 \big),\\
    &D^{l+1}:=D^l+X^{l+1}-Z^{l+1},
\end{align*}
where $X^l$ is the value of $X$ after iteration $l$. The $X$-update and $Z$-update are the evaluation of proximal operators associated with functions $f$ and $g$, respectively. 
The ADMM achieves linear convergence under some mild assumptions \cite{hong2017linear}. In practice, it has also been demonstrated that ADMM generates solutions with moderate precision within few tens of iterations \cite{boyd2011distributed}.

\subsection{ADMM for Attention Allocation}
In this section, we present an algorithm that solves the attention allocation problem (\ref{eq:LMI}) with the time complexity $\mathcal{O}(H)$.
\begin{theorem}
The attention allocation problem (\ref{eq:LMI}) can be solved with the time complexity $\mathcal{O}(H)$, where $H$ is the receding horizon.
\end{theorem}

\begin{proof}
The proof is based on constructing a  problem equivalent to (\ref{eq:LMI}) which is in the ADMM form  (\ref{eq:admm_form}), and deriving a set of ADMM updates that solves the constructed problem in $\mathcal{O}(H)$ time. 

By introducing the slack variables $S_{k|k-1} \succeq 0$ and $S_{k|k}\succeq 0$ for $k\in [t,t+H]$,  problem (\ref{eq:LMI}) can be rewritten as
\begin{subequations}
\label{eq:ADMM}
\begin{align}
\nonumber
    \min \; & \sum_{k=t}^{t+H} Tr(\Theta_k Q_{k|k}^{-1})\\ \label{eq:ADMM_obj} 
    & \quad +\frac{\beta}{2} ( Tr( Q_{k|k,j}^{-1} Q_{k|k})- \log\det Q_{k|k-1})\\ \label{eq:ADMM_cons_b}
    \text{s.t.} \quad  
    & Q_{k|k}= Q_{k|k-1}+ C_{k}^\top U_{k} C_{k}, 
    \; \forall k\in [t,t+H], \\  \nonumber
    & \begin{bmatrix}
    S_{k|k-1}& S_{k|k-1} A_{k-1}& S_{k|k-1}W^{\frac{1}{2}}\!\\
    \!A_{k-1}^\top S_{k|k-1}& S_{k-1|k-1}& 0\\
    W^{\frac{1}{2}} S_{k|k-1}& 0 & I
    \end{bmatrix}\!\! \succeq 0, \\ \label{eq:ADMM_cons_c}
    & \forall k \in [t+1,t+H] \ 
    \text{and}\  S_{t|t-1} = P_{t|t-1}^{-1}, \\ \label{eq:ADMM_cons_d}
    & U_{k} \preceq \hat{V}^{-1}, \quad \forall k\in[t, t+H],\\
    &Q_{k|k-1}= S_{k|k-1}, \quad \forall k\in[t, t+H],\\
    &Q_{k|k}= S_{k|k}, \quad \forall k\in[t, t+H],
\end{align}
\end{subequations}
where the variables are $U_k$, $Q_{k|k-1}$, $ Q_{k|k}$, $S_{k|k}$, and $S_{k|k-1}$ for $k \in [t,t+H]$. From the construction, it is clear that (\ref{eq:LMI}) and (\ref{eq:ADMM}) share the same optimum value and optimizers.

The problem (\ref{eq:ADMM}) is in the ADMM form (\ref{eq:admm_form}) with $X$ denoting the variables for $U_{k}$, $Q_{k|k}$, and $Q_{k|k-1}$, and  $Z$ denoting the variables for $S_{k|k}$ and $S_{k|k-1}$. $f(X)$ is the sum of the objective function in (\ref{eq:ADMM_obj}) and the indicator functions for (\ref{eq:ADMM_cons_b})  and (\ref{eq:ADMM_cons_d}). $g(Z)$ is the indicator function of (\ref{eq:ADMM_cons_c}).

We now construct the ADMM iterations as follows: The  $X$-update step of ADMM  at iteration $l+1$ is given by solving 
\begin{equation}
\label{eq:first_update}
\begin{split}
    \min \quad & \sum_{k=t}^{t+H} \big [ Tr(\Theta_k Q_{k|k}^{-1})\\ &\quad +\frac{\beta}{2} ( Tr(Q_{k|k,j}^{-1} Q_{k|k})- \log\det Q_{k|k-1})\\ 
     & \quad +\frac{\rho}{2} \|Q_{k|k-1} - S_{k|k-1}^l+ D_{k|k-1}^l\|_F^2\\ 
     &\quad +\frac{\rho}{2} \|Q_{k|k}-S_{k|k}^l+ D_{k|k}^l \|_F^2 \big]\\ 
    \text{s.t.} \quad & \mathrm{(\ref{eq:ADMM_cons_b})}\   \text{and} \  \mathrm{(\ref{eq:ADMM_cons_d})},
\end{split}
\end{equation}
with variables $U_k, Q_{k|k-1}$, and  $Q_{k|k}$ for $k\in [t,t+H]$, where $D_{k|k-1}$ and $D_{k|k}$ denote the dual variables for $Q_{k|k-1}=S_{k|k-1}$ and $Q_{k|k}=S_{k|k}$, respectively. This problem is separable for each time step $k$ and hence the computation is parallelizable. 
Let $Q^{l+1}_{k|k-1}$ and $ Q^{l+1}_{k|k}$ for $k \in [t,t+H]$ be the optimal solution of (\ref{eq:first_update}). The  $Z$-update is given by solving 
\begin{equation}
\label{eq:second_update}
\begin{split}
    \min \quad & \sum_{k=t}^{t+H} \big[ \|Q_{k|k-1}^{l+1}- S_{k|k-1}+ D_{k|k-1}^l\|_F^2\\
    & \quad + \|Q_{k|k}^{l+1}-S_{k|k}+ D_{k|k}^l \|_F^2 \big]\\ 
    \text{s.t.} \quad  & \mathrm{(\ref{eq:ADMM_cons_c})} ,
\end{split}
\end{equation}
with variables $S_{k|k-1}$ and $S_{k|k}$ for $k \in [t,t+H]$. Problem (\ref{eq:second_update}) is separable for each pair of $(S_{k|k-1},S_{k-1|k-1})$. Hence, the $Z$-update, like $X$-update, scales linearly with $H$. 
Let $S_{k|k-1}^{l+1}$ and $ S_{t|t}^{l+1}$ for $k \in [t:t+H]$ be the optimal solution of (\ref{eq:second_update}). The dual-update of ADMM algorithm is separated for each $k\in[t, t+H]$ and is given by
\begin{align*}
    D_{k|k-1}^{l+1}&=D_{k|k-1}^{l}+ Q_{k|k-1}^{l+1}-S^{l+1}_{k|k-1},\\
    D_{k|k}^{l+1}&=D_{k|k}^{l}+ Q_{k|k}^{l+1}-S^{l+1}_{k|k},
\end{align*}
The dual update also scales linearly with $H$. This completes the proof. 
\end{proof}

\section{Simulations}
\label{sec:simulation}

\begin{figure*}[t!]
    \centering
    %
    \subfloat[\label{subfig:traj_beta16}Trajectory of the robot: $\beta=18$]{\makebox[42mm][c]{\includegraphics[trim = 0.2cm 0cm 0.8cm 0.6cm, clip=true,
    width=4.2cm]{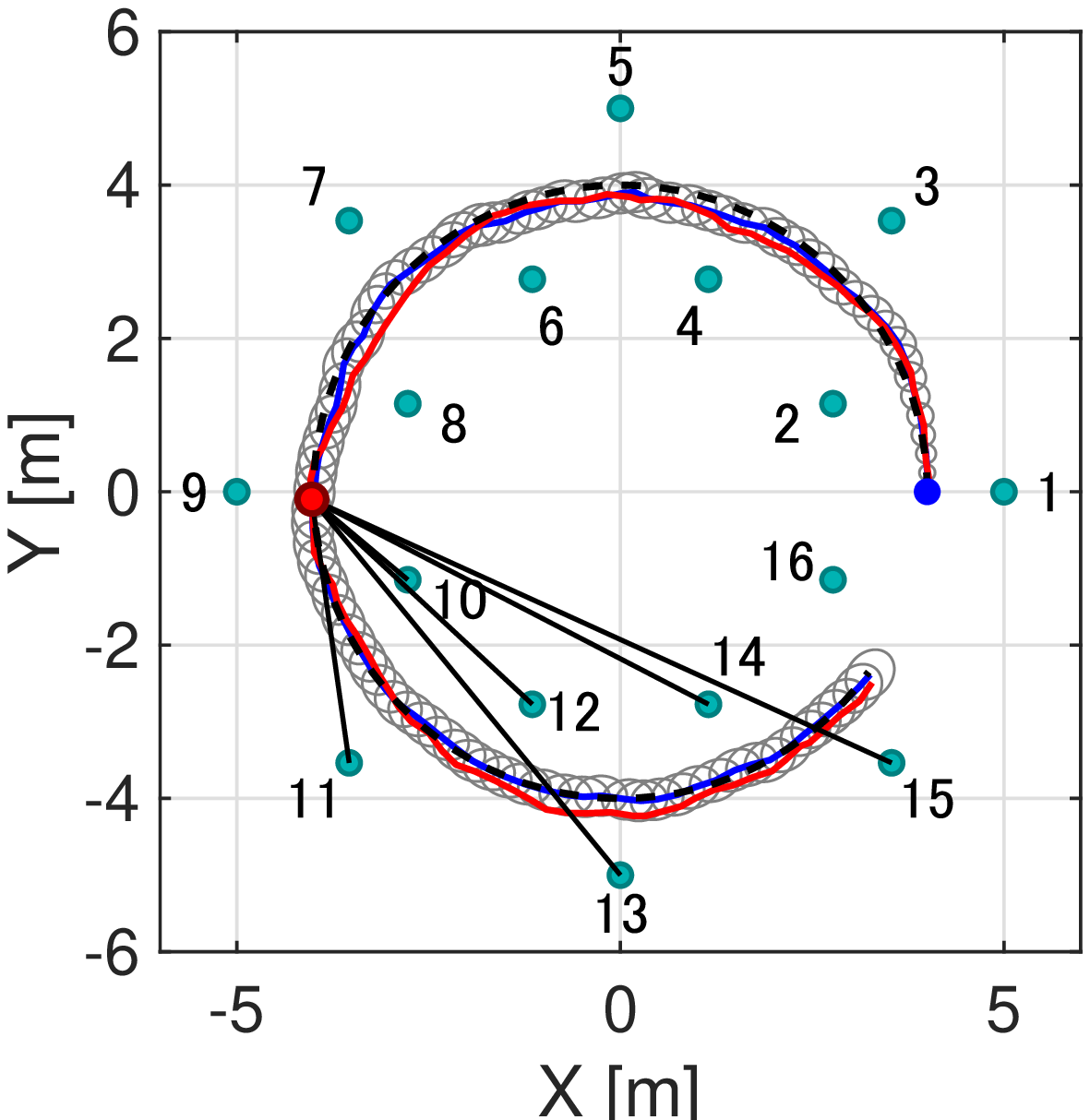}
    \label{fig:traj_beta16}}} \quad 
    \subfloat[\label{subfig:cont_beta16}Data rate allocation: $\beta=18$]{\makebox[38mm][c]{\includegraphics[trim = 0.2cm 0cm 0.8cm 0.6cm, clip=true,
    width = 3.8cm]{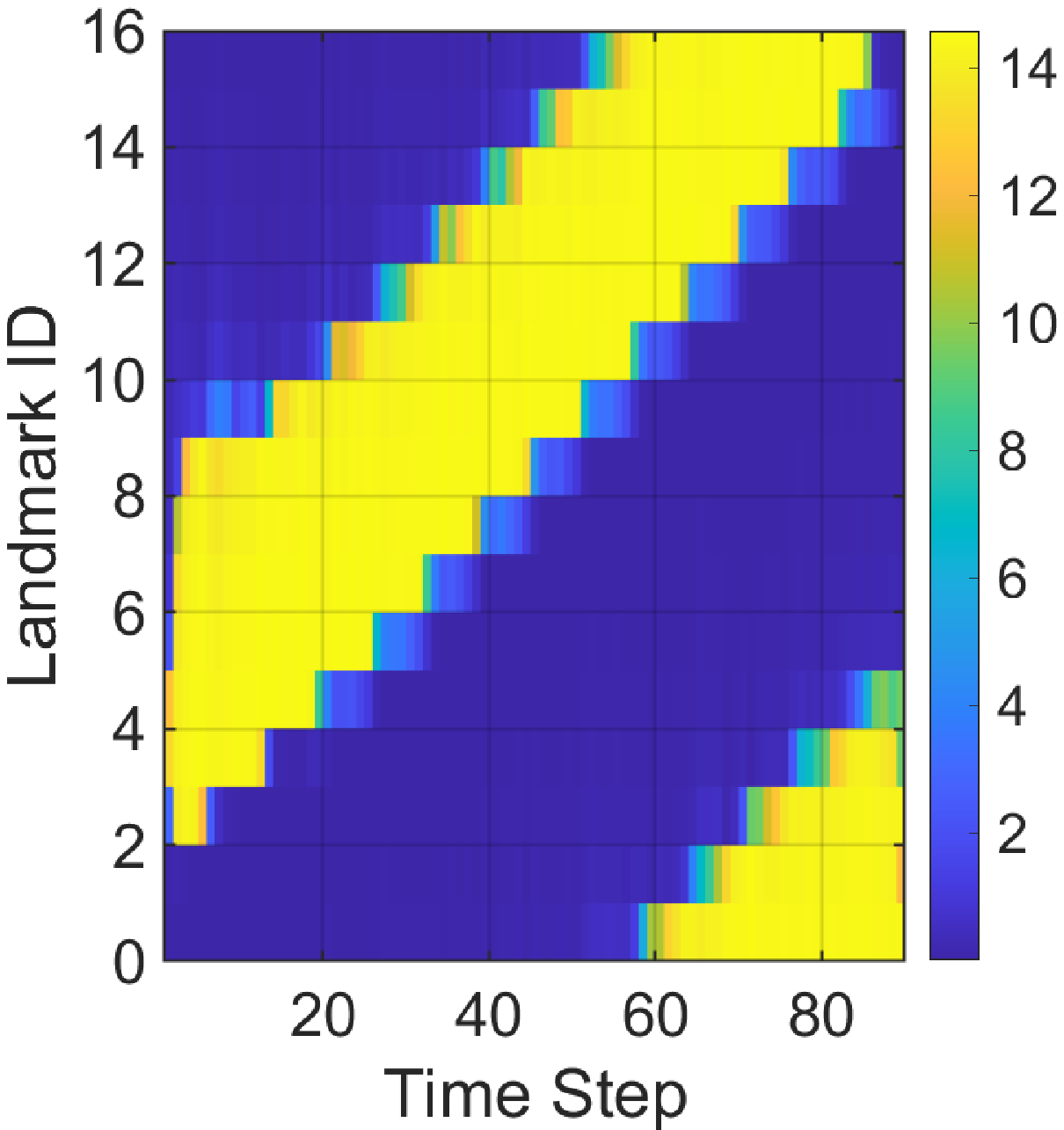}
    \label{fig:cont_beta16}}} \qquad
    \subfloat[\label{subfig:traj_beta28}Trajectory of the robot: $\beta=32$]{\makebox[42mm][c]{\includegraphics[trim = 0.2cm 0cm 0.8cm 0.5cm, clip=true,
    width=4.2cm]{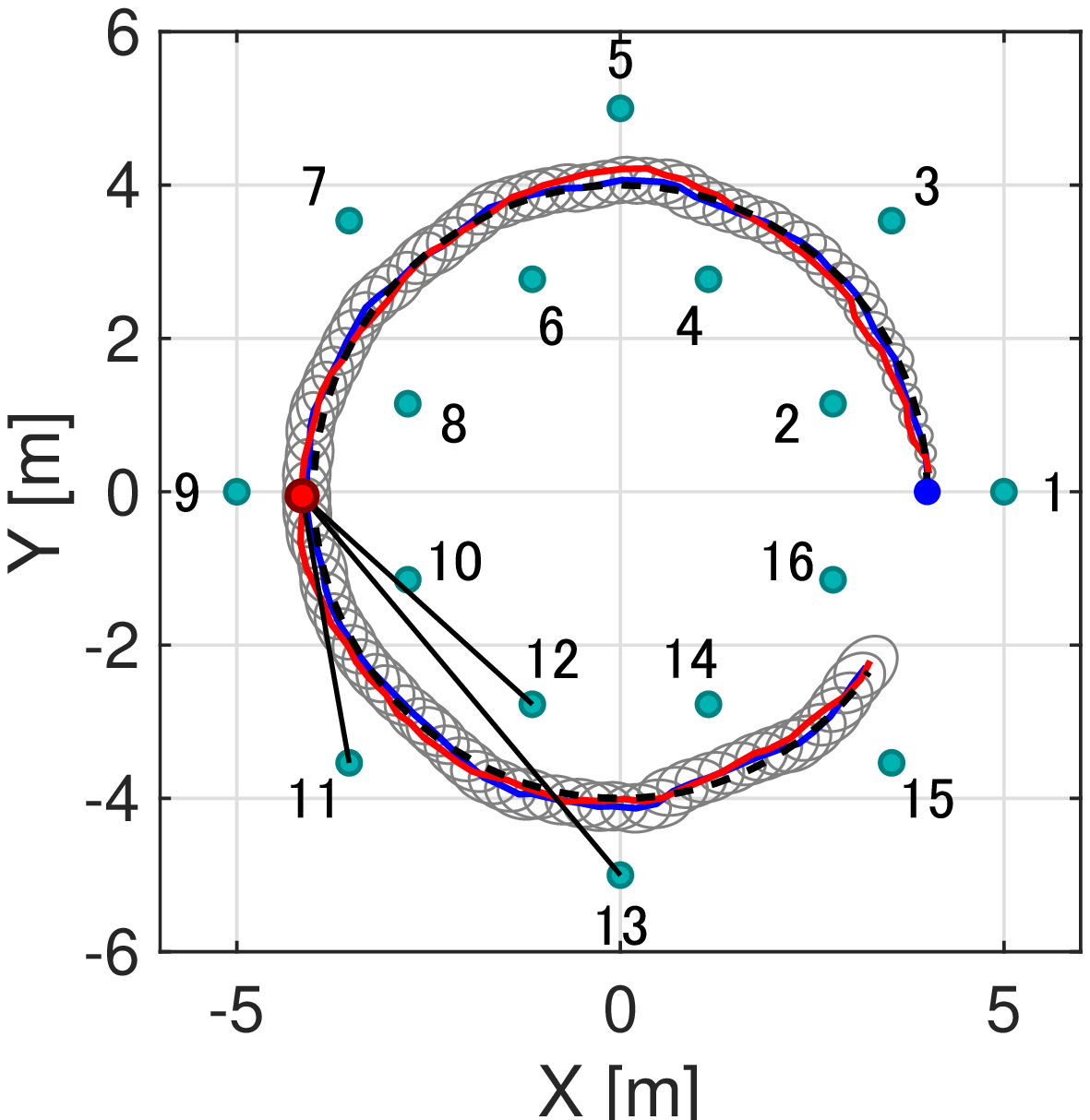}
    \label{fig:traj_beta28}}} \quad 
    \subfloat[\label{subfig:cont_beta28}Data rate allocation: $\beta=32$]{\makebox[38mm][c]{\includegraphics[trim = 0.2cm 0cm 0.8cm 0.6cm, clip=true,
    width = 3.8cm]{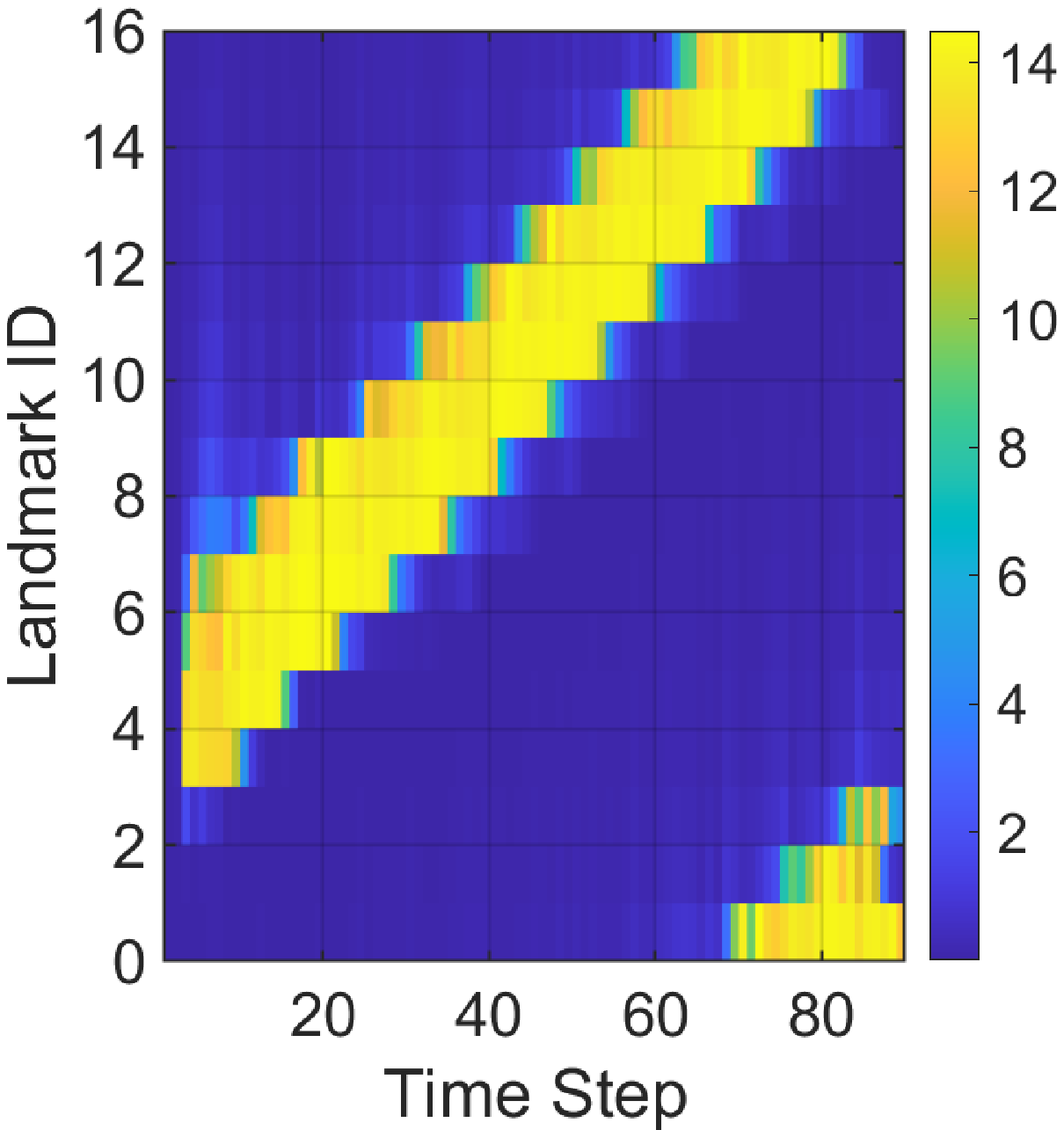}
    \label{fig:cont_beta28}}} \quad
    \caption{Results of the proposed algorithm with $H=10$, where (a)-(b) $\beta = 18$ and (c)-(d) $\beta = 32$. 
    (a) and (c): The reference trajectory is depicted as the black dashed line, while the actual robot position is shown with the red line. The trajectory of mean and the covariance ellipses representing $90\,\%$ certainty regions estimated through KF are shown with a blue line and gray ellipses, respectively. 
    The 16 landmarks are indicated by green dots with their ID.
    The red dot that appears on the left is the actual robot position at $k=50$, with black lines connecting the robot and the selected landmarks. The initial position of the robot is shown with the blue dot. (b) and (d): The allocated data rate $U_i$ for each landmark $i\in [16]$.}
    \label{fig:sim_multi_beta}
\end{figure*}

\begin{figure}[t!]
    \centering
    \subfloat[\label{subfig:traj_greedy}Trajectory of the robot]{\makebox[42mm][c]{\includegraphics[trim = 0.2cm 0cm 0.8cm 0.6cm, clip=true,
    width=4.2cm]{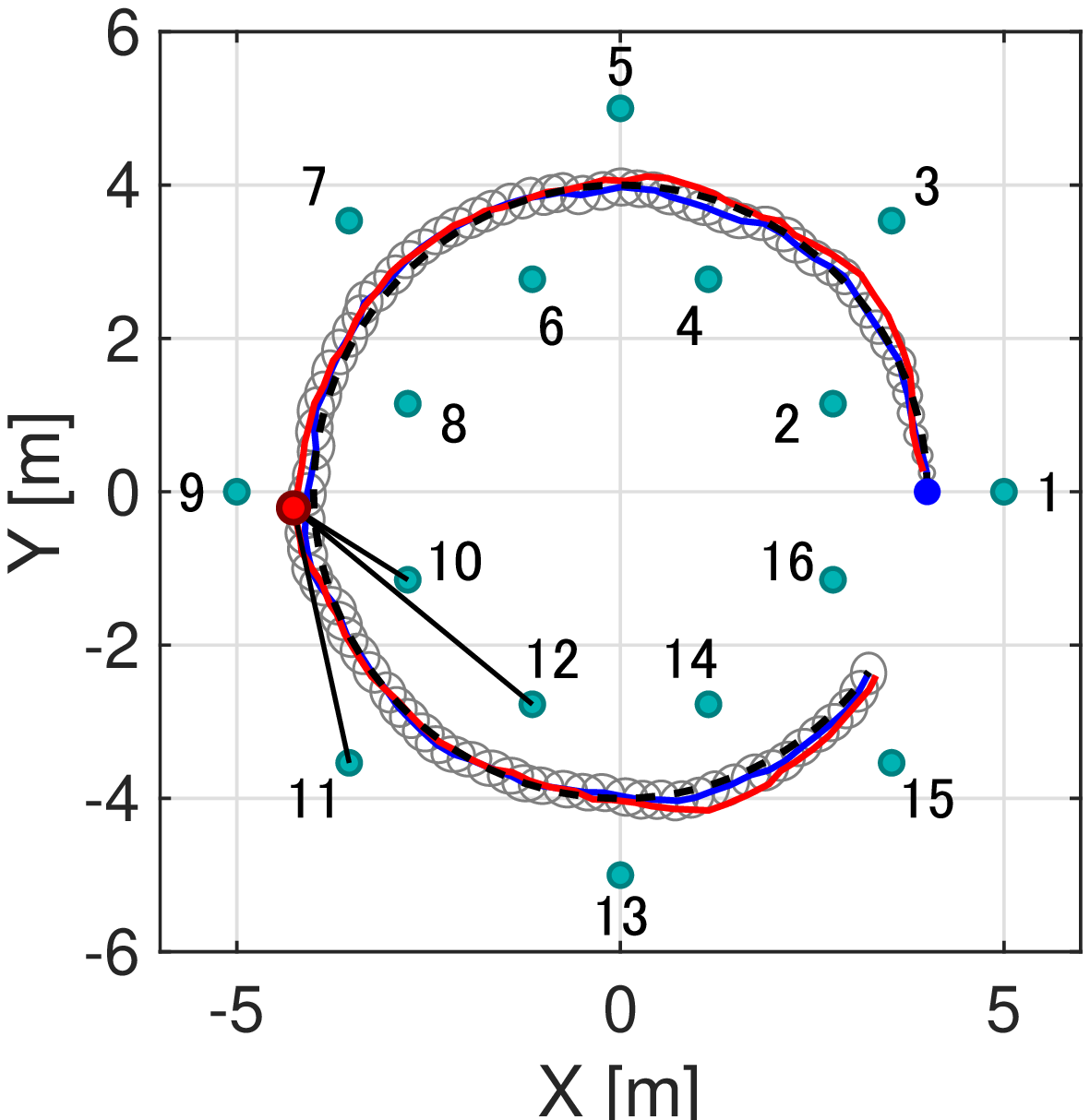}
    \label{fig:traj_RH1}}} \quad 
    \subfloat[\label{subfig:cont_greedy}Selected landmarks]{\makebox[38mm][c]{\includegraphics[trim = 0.2cm 0cm 0.8cm 0.6cm, clip=true,
    width = 3.8cm]{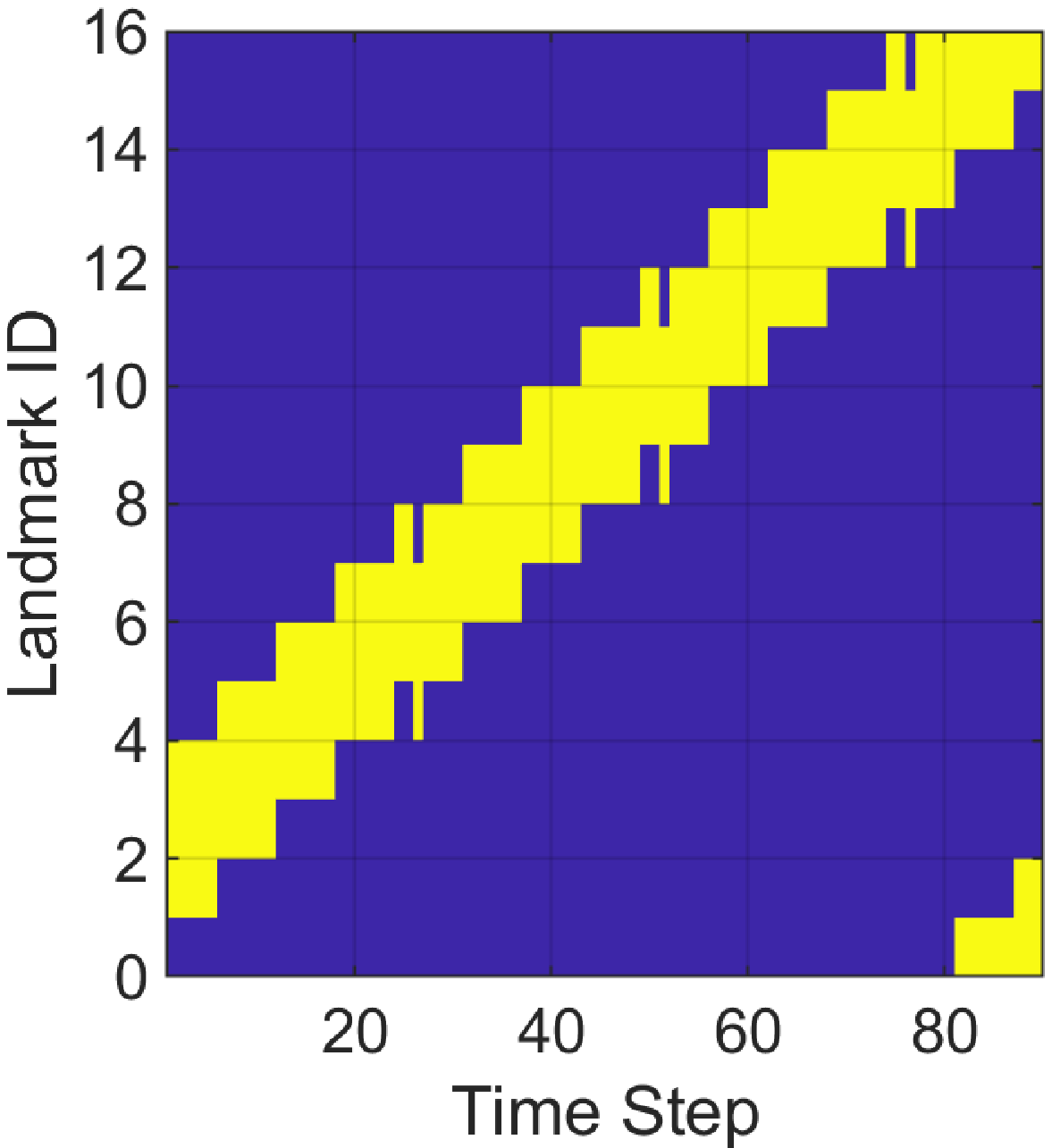}
    \label{fig:cont_RH1}}}
    \caption{Results of the greedy algorithm. Every time step, the algorithm selects three landmarks.}
    \label{fig:sim_greedy}
\end{figure}

\begin{figure}[t!]
    \centering
    \subfloat[\label{subfig:traj_HR}Trajectory of the robot]{\makebox[42mm][c]{\includegraphics[trim = 0.2cm 0cm 0.8cm 0.6cm, clip=true,
    width=4.2cm]{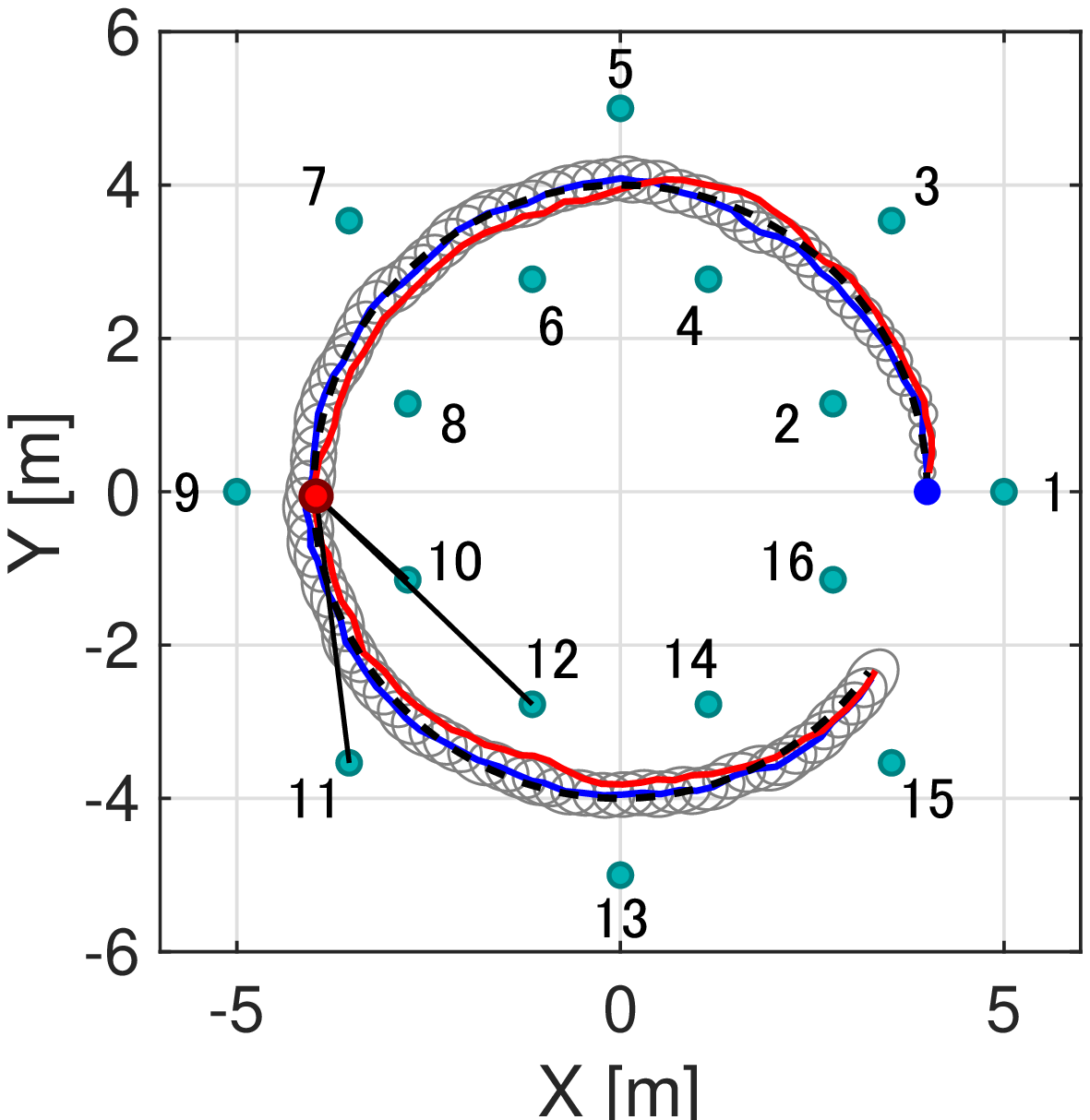}
    \label{fig:traj_greedy}}} \quad 
    \subfloat[\label{subfig:cont_HR}Data rate allocation]{\makebox[38mm][c]{\includegraphics[trim = 0.2cm 0cm 0.8cm 0.6cm, clip=true,
    width = 3.8cm]{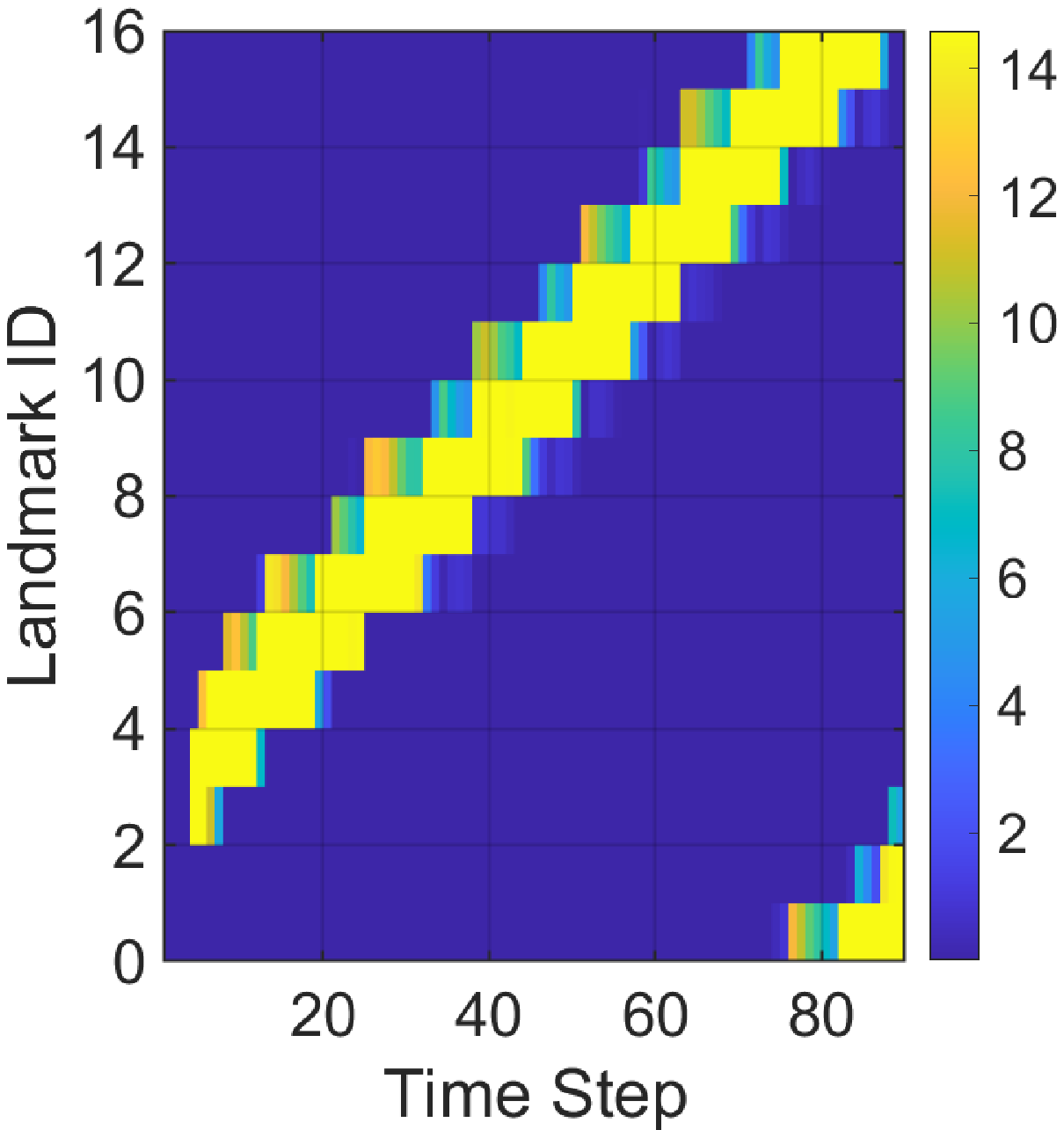}
    \label{fig:cont_greedy}}}
    \caption{Results of the proposed algorithm with $\beta = 2.5$, $H=0$ and $\hat V_i^{-1}~=~14.6$.}
    \label{fig:sim_RH0}
\end{figure}

\begin{figure}[t!]
    \centering
    \subfloat[\label{subfig:large_U_traj}Trajectory of the robot]{\makebox[42mm][c]{\includegraphics[trim = 0.2cm 0cm 0.8cm 0.6cm, clip=true,
    width=4.2cm]{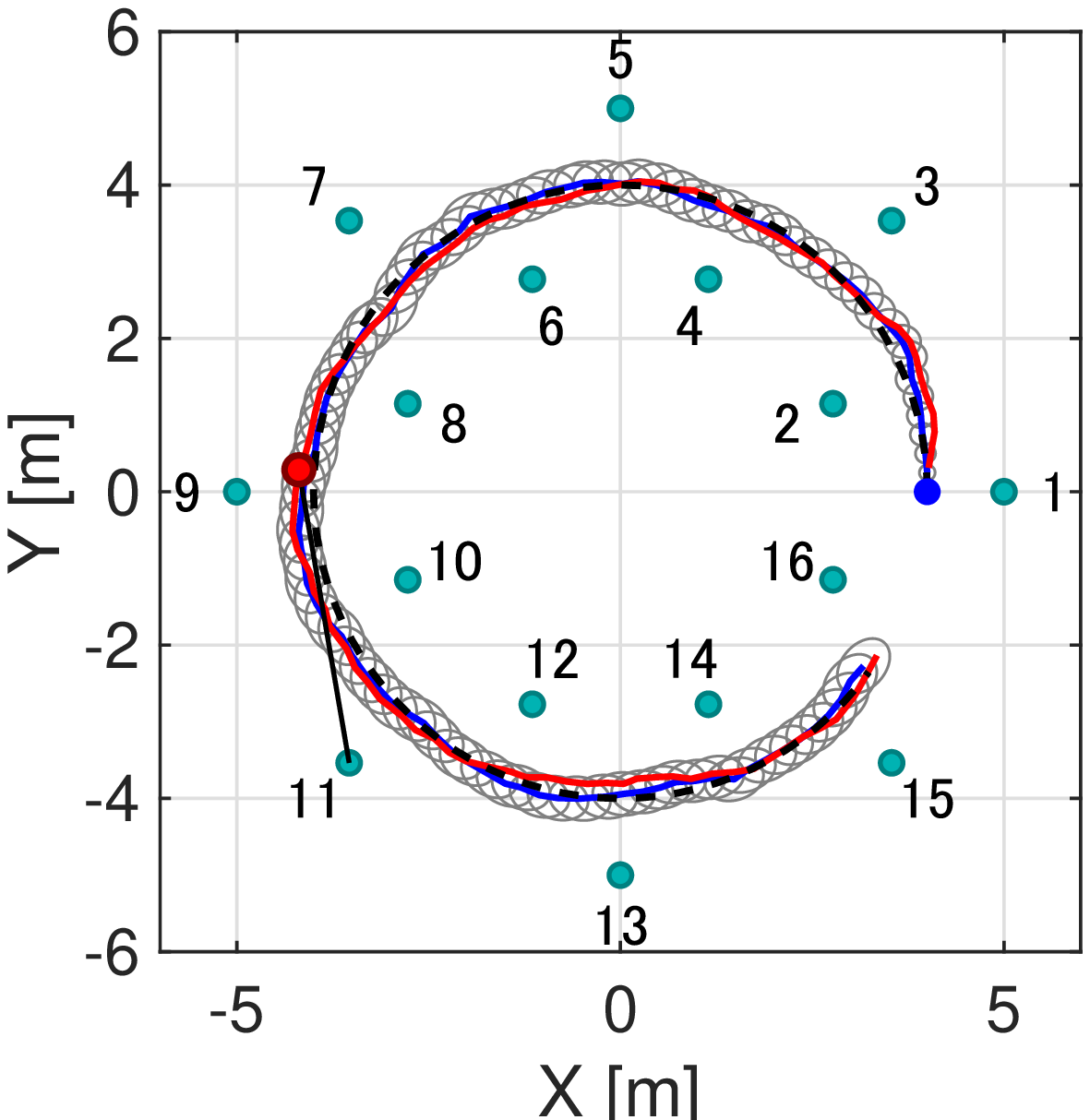}
    \label{fig:large_U_traj}}} \quad 
    \subfloat[\label{subfig:large_U_cont}Data rate allocation]{\makebox[38mm][c]{\includegraphics[trim = 0.2cm 0cm 0.8cm 0.6cm, clip=true,
    width = 3.8cm]{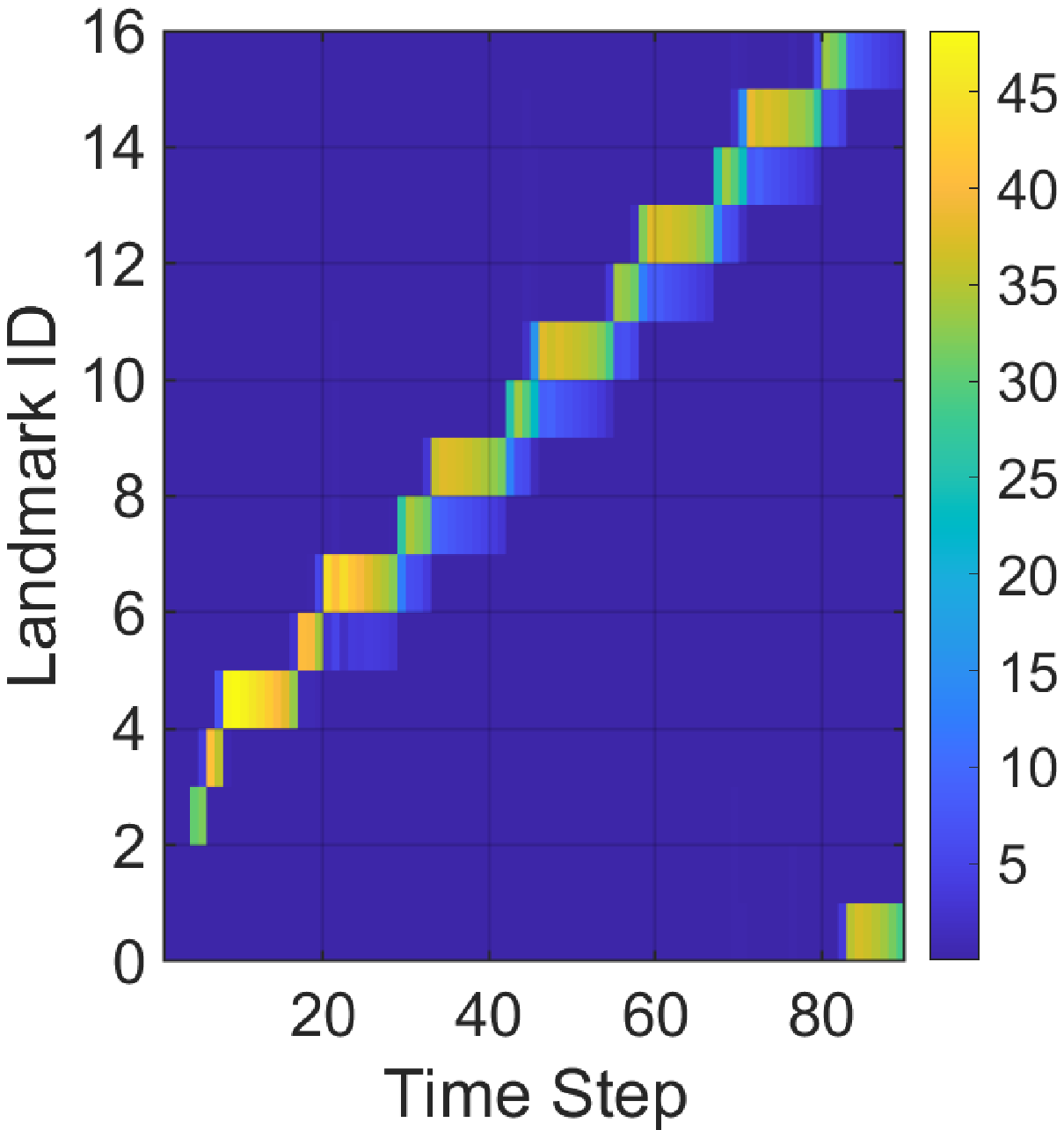}
    \label{fig:large_U_cont}}}
    \caption{Results of the proposed algorithm with $\beta = 2.5$, $H=0$, and $\hat V_i^{-1}=270$. In (a), the landmark allocated the largest data rate, $13\,\%$ of $\hat V_i^{-1}$, at $k=50$ is connected with the robot shown with the red dot by the black line. In (b), though the maximum data rate allowed has a much higher value 
    than those allocated during the simulation, we adjust the range of the color map so that the landmark allocated relatively large data rate are emphasized.}
    \label{fig:large_U}
\end{figure}

The proposed attention allocation strategy is simulated to investigate the impact of the data rate constraint.
Here, we only implement the algorithm presented in Section~\ref{sec:proposed_form} because the length of the horizon $H$ utilized is small.


\subsection{Simulation Setup} \label{ssec:sim_set}

We consider a scenario in which a mobile robot equipped with an omnidirectional camera follows a given reference trajectory
by measuring relative angles between itself and known landmarks. The state of the robot at time step $k$ is the 2-D position $[x_k~y_k]^T$ and the orientation $\theta_k$. 
Its motion is governed by the unicycle model perturbed with a Gaussian i.i.d. noise
\begin{equation*}
   \begin{bmatrix}
    \bm{x}_{k+1} \\ \bm{y}_{k+1} \\ \bm{\theta}_{k+1}
    \end{bmatrix} \!=\! 
    \begin{bmatrix}
    \bm{x}_k \\ \bm{y}_k \\ \bm{\theta}_{k}
    \end{bmatrix} \!+\!
    \begin{bmatrix}
    v_k \cos{\bm{\theta}_k} \\ v_k \sin{\bm{\theta}_k} \\ \omega_k
    \end{bmatrix} \Delta t +  \bm{w}_k,~\bm{w}_k \sim \mc N(0, W),
\end{equation*}
with the velocity and angular velocity input $u_k = [v_k~\omega_k]^T$.
We set $W = {\rm diag}(1.2, 1.2, 30)\times 10^{-3}$.

%
The visual information of landmarks 
is captured through an omnidirectional camera mounted on the robot. 
We assume that the relative angles between landmarks and the robot are obtained 
through the computer vision techniques \cite{Lowry16_visual_place} and well-known camera model \cite{Kawai11_panorama_cam}.
The measurement model can be modeled as
\begin{equation*}
    \bm{y}_{k} =  
    \begin{bmatrix}
    \arctan{\left( \frac{m_{1,y}-\bm{y}_k}{m_{1,x}-\bm{x}_k} \right)} - \bm{\theta}_k
    \\ \vdots \\
    \arctan{\left( \frac{m_{M,y}-\bm{y}_k}{m_{M,x}-\bm{x}_t} \right)} - \bm{\theta}_k
    \end{bmatrix}
    + \bm{v}_k,~ \bm{v}_k \sim \mc N(0, \hat V),
\end{equation*}
with positions of known landmarks $m_j = [m_{j,x}~m_{j,y}]^T$ for  $j \in [M]$, 
where $\hat V =  {\rm diag}(\{\hat{V}_{i}\}_{i \in [M]})$ is the noise level of the landmarks in case the robot allocates the full attention to them.
We set $\hat V_i^{-1} = 14.6,~\forall i\in [M]$, which corresponds with the robot can measure the relative angle with accuracy up to a standard variance $15$\,deg.
A scenario with more accurate sensors (with large $\hat{V}^{-1}$) is considered in Section~\ref{ssec:sim_greedy}.



The reference trajectory is depicted as the black dashed line in Fig. \ref{fig:sim_multi_beta}(a). It draws a circle with the radius of 4\,m starting from $[4~0~\pi/2]^\top$, moving in an anticlockwise direction, and ending at the lower right part at $T=90$.
The reference points are evenly spaced, and the reference direction $\theta_t^{ref}$ is the same as that of the tangent line of the circle. 
We placed 16 landmarks along the circle as indicated by green dots with their IDs.
The initial position of the robot is shown with the blue dot.

We employ the control and the state estimation scheme introduced in Section~\ref{sec:EKF} and \ref{sec:LQR}. 
For all the simulations, we set $Q = {\rm diag}(0.3, 0.3, 1.6)$ and $R = {\rm diag}(3.5, 3.5)\times~10^{-3}$. 

\subsection{Impact of the DI cost with various weight $\beta$}
We first apply the proposed method with $H = 10$ and $\beta = 18, 32$ to illustrate the effects of varying the data rate cost on the proposed data rate allocation strategy. 
Fig. \ref{fig:sim_multi_beta}(a) and (c) show the trajectory of a robot in a red line, while the mean and the covariance estimated through KF are depicted as a blue line and gray ellipses for each simulation. 
The allocated data rate for each landmark $j \in [16]$ is illustrated as contour maps in Fig. \ref{fig:sim_multi_beta}(b) and (d).


From Fig.~\ref{fig:sim_multi_beta}(b) and (d), we observe that the proposed method for both $\beta = 18, 32$ tends to allocate either zero or the maximum data rate, namely $\hat V_i^{-1}$, for most landmarks. 
Furthermore, as $\beta$ is increased to $32$, the number of landmarks allocated the large data rate becomes smaller. 
In Fig.~\ref{fig:sim_multi_beta}(a) and (c), the landmarks allocated a large data rate at $k=50$ are connected with the robot, shown with the red dot, by the black lines.

\subsection{Comparison with greedy selection} \label{ssec:sim_greedy}

To better understand the characteristics of the proposed attention allocation strategy, we compare the proposed method and the greedy algorithm based on \cite{Hashemi21_greedy_sensor} with the simulation setting same as the above.
We add the minor modification to the greedy algorithm \cite{Hashemi21_greedy_sensor} so that it minimizes the LQG cost $C_k= Tr(\Theta_k Q_{k|k}^{-1})$.
The number of landmarks to be selected at each time step is set to three.
Note that the same control and state estimation scheme to the proposed algorithm is employed.

Since the greedy algorithm does not have a receding horizon policy, we set $H = 0$ and $\beta = 2.5$ for the proposed algorithm to make a fair comparison. 
The result of the greedy algorithm and the proposed method are shown in Fig.~\ref{fig:sim_greedy} and Fig.~\ref{fig:sim_RH0}, respectively. 
First, we confirm that both methods focus on almost the same landmarks at most time steps. 
Second, the proposed algorithm does not choose any landmarks at the initial three steps as opposed to the greedy algorithm. 
This is because the initial covariance is already small, and hence trying to shrink it further at the expense of large data rate cost is not reasonable. A similar result can also be observed in Fig. \ref{fig:sim_multi_beta}.


Another clear difference between the proposed method and the greedy algorithm is the freedom of allocating moderate attention to the landmarks. 
To demonstrate how this capability affects the strategy, we apply the proposed algorithm to the setting where the robot is mounted with a more accurate visual sensor. 
Here, we set $\hat V_i^{-1} = 270,~\forall i\in [M]$, which means the robot can obtain the relative angle with a standard variance $3.5$\,deg.

Fig. \ref{fig:large_U} shows the simulation result, where the data allocation strategy completely changes from that of Fig. \ref{fig:sim_RH0} using a low-resolution sensor, although we do not change the parameters in the proposed method.
Fig. \ref{fig:large_U}(b) shows that the proposed method does not allocate the full capacity $\hat V_i^{-1} = 270,~\forall i\in [M]$ to any of the landmarks and focuses on only one landmark for most of the time steps. 
The highest data rate allocated during the simulation is $48.2$, only $18\,\%$ of $\hat V_i^{-1}$.
Even with the differences in the data rate allocation, both simulations show the equivalent tracking performance with almost the same size of covariance ellipses in Fig. \ref{fig:sim_RH0}(a) and Fig. \ref{fig:large_U}(a).
This means that the proposed algorithm can adjust its strategy to the capability of the given sensors. 
This adaptability does not appear in the greedy algorithm as it can assign only the full or zero data rate.

\section{Discussion}
\label{sec:discus}


As observed in Section~\ref{sec:simulation}, our formulation tends to admit sparse solutions, i.e., solutions with many entries such that $U_{i,t}=0$. In this section, we develop an intuition as to why (\ref{eq:prob_new}) promotes sparsity by considering a simple special case with scalar time-invariant system for which a closed-form solution is available.

Consider a scalar time-invariant system is described by
\[\bm{x}_{t+1}= a \bm{x}_{t}+ b \bm{u}_t + \bm{w}_t, \quad \bm{w}_t \overset{i.i.d.}{\sim} \mathcal{N}(0, W),\]
\[\bm{y_t}=\bm{x}_t+\bm{v_t},\quad \bm{v}_t \overset{i.i.d.}{\sim} \mathcal{N}(0, \hat{V}),\]
where $\bm{x}_t \in \mathbb{R}$, and $\bm{y}_t \in \mathbb{R}$. The  infinite-horizon limit  of (\ref{eq:relaxed_prob}) for this system is formulated as 
 \begin{subequations}
 \label{eq:def_SRA}
 \begin{align}
     \min_V \quad & \limsup_{t\rightarrow \infty} \frac{1}{T} \big [J_{1:T} + \beta \sum_{t=1}^T I(\bm{x}_t; \bm{y}_t|\bm{y}_{1:t-1})  \big ]\\
     \text{s.t.} \quad & \hat{V} \leq V,
 \end{align}
 \end{subequations}
where $V \triangleq \limsup_{t\rightarrow \infty} V_t$. Since the system is observable, $P\triangleq \lim_{t \rightarrow \infty} P_{t|t}$  exists and is computed by the algebraic Riccati equation (ARE) 
\begin{equation}
\label{eq:ARE}
    P^{-1}= (a^2P+W)^{-1}+V^{-1}.
\end{equation}
It is elementary to verify that the stationary problem (\ref{eq:def_SRA}) can be explicitly written as
 \begin{subequations}
 \label{eq:SRA}
 \begin{align}
     \min_{P,V} \quad & \theta P + \frac{\beta}{2}\log (a^2+\frac{W}{P})\\
     \text{s.t} \quad &P^{-1}= (a^2P+W)^{-1}+V^{-1}\\
     & \hat{V} \leq V_{\infty} \leq \infty,
 \end{align}
 \end{subequations}
 where $\theta \triangleq \lim_{t \rightarrow \infty} \theta_t$. Denote by $P=g(V)$ the unique positive solution to the ARE (\ref{eq:ARE}). It is easy to show that $g(V)$ is a strictly increasing function of $V$. Therefore, the problem (\ref{eq:SRA}) can be equivalently written as 
 
\begin{subequations}
\label{eq:equi_prob}
 \begin{align}
 \label{eq:equi_prob_a}
     \min_{P} \quad &  \theta P + \frac{\beta}{2}\log (a^2+\frac{W}{P})\\
     \text{s.t} \quad & g(\hat{V}) \leq P \leq  g(\infty)=\frac{W}{1-a^2} 
 \end{align}
 \end{subequations}

The convex objective function in (\ref{eq:equi_prob_a}) has a unique minimizer $P(\beta):=\frac{-\theta W+\sqrt{(\theta W)^2+4a^2\theta W \beta}}{2\theta a^2}  $ in $\mathbb{R}^{+}$ domain. Let $\beta_1$ and $\beta_2( > \beta_1)$ be the values of $\beta$ such that $P(\beta_1)= g(\hat{V})$ and $P(\beta_2)=\frac{W}{1-a^2}$, respectively. Then, the optimal solution for (\ref{eq:equi_prob}) is:
\begin{equation}
\label{eq:sol}
    P^*=\begin{cases}
    g(\hat{V})  & \text{if} \  \beta \leq \beta_1,\\
    \frac{-\theta W+\sqrt{(\theta W)^2+4a^2\theta W\beta}}{2\theta a^2} &  \text{if} \ \beta_1 < \beta \leq \beta_2,\\
    g(\infty)=\frac{W}{1-a^2} &  \text{if} \  \beta_2 < \beta.
    \end{cases}
\end{equation}
The first case happens when the weight on directed information is small and the DM decides to make the full measurement $V=\hat{V}$. The third case happens when the weight on the directed information is high and the DM decides to make no measurement. Therefore, the intermediate choice $\hat{V}< V<\infty$ of the sensing gain only occurs when $\beta_1<\beta<\beta_2$. This partly explains the sparsity promoting phenomenon.



Although the analytical method discussed above is not applicable to (\ref{eq:prob_new}) in full generality, the sparsity-promoting property of the regularization with DI may be understood by invoking its mathematical similarities to other sparsity promoting regularizers widely known in the literature \cite{candes2008enhancing}. This is postponed as our future work.






\section{CONCLUSIONS AND FUTURE WORKS}

In this paper, we studied the problem of landmark selection under a constraint on the data rate of the information flow coming from the observation. We formulated the problem as finding an optimal data rate assignment that minimizes the weighted sum of the control cost and the DI between DM's state and the observations. We showed this problem can be reformulated as a DC program, and we used the CCP algorithm to find an optimizer. We reduced the computation time of CCP by developing a scalable distributed algorithm based on ADMM. The algorithms were tested in trajectory tracking simulations, where the sparsity-promoting nature of formulation was observed. We examined the sparsity-promoting property by solving the specific instance of the scalar system for the infinite horizon limit. 





\bibliographystyle{IEEEtran}
\bibliography{ref.bib}

\end{document}